\documentclass[fleqn,10pt]{wlpeerj}
\usepackage{amsmath, amssymb, amsthm}
\usepackage{mathtools}
\usepackage[ruled,vlined]{algorithm2e}

\title{Optimal 1-NN Prototypes for Pathological Geometries}

\author[1]{Ilia Sucholutsky}
\author[2]{Matthias Schonlau}
\affil[1,2]{Department of Statistics and Actuarial Science, University of Waterloo}
\corrauthor[1]{Ilia Sucholutsky}{isucholu@uwaterloo.ca}

\DeclarePairedDelimiter\ceil{\lceil}{\rceil}

\newtheorem{theorem}{Theorem}

\newtheorem{corollary}[theorem]{Corollary}

\newcounter{subdefinition}[theorem]
\renewcommand{\thesubdefinition}{\thedefinition.\arabic{subdefinition}}

\usepackage{xpatch}
\makeatletter
\AtBeginDocument{\xpatchcmd{\@thm}{\thm@headpunct{.}}{\thm@headpunct{}}{}{}}
\makeatother
\begin{abstract}
Using prototype methods to reduce the size of training datasets can drastically reduce the computational cost of classification with instance-based learning algorithms like the k-Nearest Neighbour classifier. The number and distribution of prototypes required for the classifier to match its original performance is intimately related to the geometry of the training data. As a result, it is often difficult to find the optimal prototypes for a given dataset, and heuristic algorithms are used instead. However, we consider a particularly challenging setting where commonly used heuristic algorithms fail to find suitable prototypes and show that the optimal prototypes can instead be found analytically. We also propose an algorithm for finding nearly-optimal prototypes in this setting, and use it to empirically validate the theoretical results. 
\end{abstract}

\begin{document}

\flushbottom
\maketitle
\thispagestyle{empty}

\section{Background}

The k-Nearest Neighbour (kNN) classifier is a simple but powerful classification algorithm. There are numerous variants and extensions of kNN \citep{dudani1976distance,yigit2015abc,sun2016stabilized,kanjanatarakul2018evidential,gweon2019k}, but the simplest version is the 1NN classifier which assigns a target point to a class based only on the class of its nearest labeled neighbor. Unfortunately, the family of kNN classifiers can be computationally expensive when working with large datasets, as the nearest neighbors must be located for every point that needs to be classified. This has led to the development of prototype selection and generation methods which aim to produce a small set of prototypes that represent the training data \citep{bezdek2001nearest,  triguero2011taxonomy, Bien_2011, garcia2012prototype, kusner2014stochastic}. Using prototypes methods speeds up the kNN classification step considerably as new points can be classified by finding their nearest neighbors among the small number of prototypes. The number of prototypes required to represent the training data can be several orders of magnitude smaller than the number of samples in the original training data. \cite{sucholutsky2020less} showed that by assigning label distributions to each prototype, the number of prototypes may even be reduced to be less than the number of classes in the data. This result was demonstrated on a synthetic dataset consisting of $N$ concentric circles where the points on each circle belong to a different class. The authors found that commonly used prototype generation methods failed to find prototypes that would adequately represent this dataset, suggesting that the dataset exhibits pathological geometries. Further analysis revealed that the soft-label kNN variant required only a fixed number of prototypes to separate any number of these circular classes, while the number of prototypes required by 1NN was shown to have an upper bound of about $t\pi$ for the $t^{th}$ circle as can be seen in Figure~\ref{fig:4circles}. However, this upper bound did not account for the possibility of rotating prototypes on adjacent circles as a method of reducing the number of required prototypes. We explore this direction to analytically find tighter bounds and an approximate solution for the minimal number of prototypes required for a 1-Nearest Neighbor classifier to perfectly separate each class after being fitted on the prototypes. In particular, we show that this problem actually consists of two sub-problems, or cases, only one of which is closely approximated by the previously proposed upper bound. We also propose an algorithm for finding nearly-optimal prototypes and use it to empirically confirm our theoretical results.

\begin{figure}
    \centering
    \includegraphics[width=0.9\textwidth]{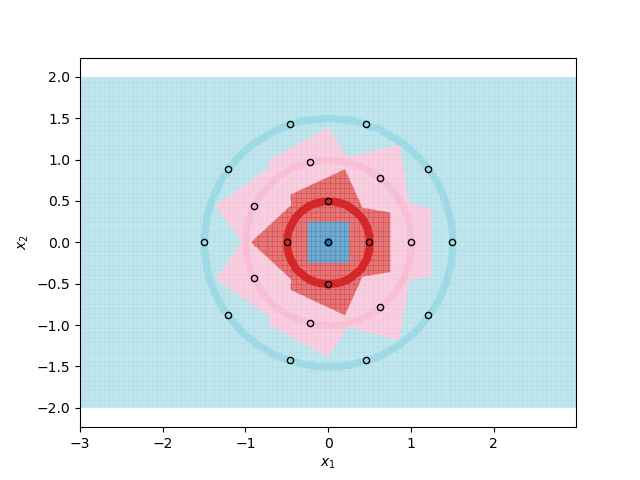}
    \caption{1NN decision boundaries when fitted on $\ceil{t\pi}$ prototypes per class. Each shaded circle represents a different class and the outlined points represent the assigned prototypes. The colored regions correspond to the decision boundaries created by the 1NN classifier.}
    \label{fig:4circles}
\end{figure}
\section{Theory}
\subsection{Preliminaries}
We first proceed to formalize the problem of having a 1-NN classifier separate the classes after being fitted on a minimal number of prototypes. Consistent with \cite{sucholutsky2020less}, we define the $t^{th}$ circle as having radius $tc$ for $t=0,1,\ldots$. Because each class is fully separated from non-adjacent classes by its adjacent classes, it is sufficient to consider arbitrary pairs of adjacent classes when trying to find the optimal prototypes. For the rest of this section, we consider arbitrarily selected circles $t$ and $t+1$ with the following radii. 

\begin{align*}
    r_1 = tc,\; r_2 = (t+1)c,\; t \in \mathbb{N}_0,\; c\in \mathbb{R}_{>0},\;
\end{align*}
Because of the symmetry of each circle, we require that the prototypes assigned to each circle be spaced evenly around it. We assume that circle $t$ and  $t+1$ are assigned $m$ and $n$ prototypes respectively. We define $\theta^*$ as the angle by which the prototypes on circle $t+1$ are shifted relative to the prototypes on circle $t$. We record the locations of these prototypes in Cartesian coordinates.

\begin{align*}
a_i &= (r_1\cos(\frac{2\pi i}{m}), r_1\sin(\frac{2\pi i}{m})), \; i=1,...,m \\
b_j &= (r_2\cos(\frac{2\pi j}{n}+\theta^*), r_2\sin(\frac{2\pi i}{n}+\theta^*)), \; i=1,...,n
\end{align*}
We can then find the arc-midpoints of these prototypes as follows.

\begin{align*}
a^*_i &= (r_1\cos(\frac{2\pi i + \pi}{m}), r_1\sin(\frac{2\pi i +\pi}{m})), \; i=1,...,m \\
b^*_j &= (r_2\cos(\frac{2\pi j + \pi}{n}+\theta^*),\ r_2\sin(\frac{2\pi i+ \pi}{n}+\theta^*)), \; i=1,...,n 
\end{align*}
Letting $d(x,y)$ be the Euclidean distance between points $x$ and $y$, we find the distances between prototypes on the same circle.

\begin{align*}
    d_a(m) &= d(a_i,a^*_i) = \sqrt{2t^2c^2 - 2t^2c^2\cos(\frac{\pi}{m})}\\
    d_b(n) &= d(b_i,b^*_i) = \sqrt{2(t+1)^2c^2 - 2(t+1)^2c^2\cos(\frac{\pi}{n})}
\end{align*}
We also find the shortest distance between prototypes of circle $t$ and arc-midpoints of circle $t+1$ and vice-versa.

\begin{align*}
    d^*_1(m,n,\theta^*) &= \min_{i,j}\{ d(a_i, b^*_j) \| i=1,...,m, \; j=1,...,n \}\\ 
    & = \min_{i,j}\{ \sqrt{t^2c^2+(t+1)^2c^2 -2t(t+1)c^2\cos(\frac{2\pi i}{m} - \frac{2\pi j +\pi}{n} - \theta^*)} \| i=1,...,m, \; j=1,...,n \}
\end{align*}
\begin{align*}
    d^*_2(m,n,\theta^*) &= \min_{i,j}\{ d(a^*_i, b_j) \| i=1,...,m, \; j=1,...,n \}\\ 
    & = \min_{i,j}\{ \sqrt{t^2c^2+(t+1)^2c^2 -2t(t+1)c^2\cos(\frac{2\pi i+\pi}{m} - \frac{2\pi j }{n} - \theta^*)} \| i=1,...,m, \; j=1,...,n \}
\end{align*}
The necessary and sufficient condition for the 1-NN classifier to achieve perfect separation is that the distance between prototypes and arc-midpoints assigned to the same circle, be less than the minimal distance between any arc-midpoint of that circle and any prototype of an adjacent circle. This must hold for every circle. Given these conditions and some fixed number of prototypes assigned to the $t^{th}$ circle, we wish to minimize $n$ by optimizing over $\theta^*$.

\begin{align*}
\textrm{Given } & m, t \; \min_{\theta^*} n\\
\textrm{s.t. }  & d^*_1(m,n,\theta^*)>d_b(n) \\
&   d^*_2(m,n,\theta^*) >d_a(m)
\end{align*}
Inspecting the inequalities, we see that they can be reduced to the following system which we note is now independent of the constant $c$ .

\begin{align}
    -\frac{2t+1}{2(t+1)} &> t\cos(\frac{2\pi i}{m} - \frac{2\pi j +\pi}{n} - \theta^*) - (t+1)\cos(\frac{\pi}{n}) \label{eqn:1}\\
    \frac{2t+1}{2t}&>(t+1)\cos(\frac{2\pi i+\pi}{m} - \frac{2\pi j }{n} - \theta^*) - t\cos(\frac{\pi}{m})\label{eqn:2}
\end{align}
It is clear that $n\ge m$, but we separate this system into two cases, $n=m$ and $n>m$, as the resulting sub-problems will have very different assumptions and solutions. The simpler case is where every circle is assigned the same number of prototypes; however, the total number of circles must be finite and known in advance. In the second case where larger circles are assigned more prototypes, we assume that the number of circles is countable but not known in advance.
We also note that for $t=0$, a circle with radius $0$, exactly one prototype is required. Given this starting point, it can be trivially shown that for $t=1$, a minimum of four prototypes are required to satisfy the conditions above (three if the strict inequalities are relaxed to allow equality). However for larger values of $t$, careful analysis is required to determine the minimal number of required prototypes. 

\subsection{Upper bounds}
We first show how our setup can be used to derive the upper bound that was found by \cite{sucholutsky2020less}.
\begin{theorem}[Previous Upper Bound]
The minimum number of prototypes required to perfectly separate $N$ concentric circles is bounded from above by approximately $\sum_{t=1}^Nt\pi$, if each circle can have a different number of assigned prototypes.
\end{theorem}
\begin{proof}
Given the setup above, we first consider the worst case scenario where a $\theta^*$ is selected such that $\cos(\frac{2\pi i}{m} - \frac{2\pi j +\pi}{n} - \theta^*) = \cos(\frac{2\pi i+\pi}{m} - \frac{2\pi j }{n} - \theta^*) = cos(0) =1$. We can then solve Inequality~\ref{eqn:1} for $n$ and Inequality~\ref{eqn:2} for $m$.

\begin{align*}
    -\frac{2t+1}{2(t+1)} &> t\cos(0) - (t+1)\cos(\frac{\pi}{n}) \\
    \cos(\frac{\pi}{n}) &> \frac{2(t+1)^2-1}{2(t+1)^2}\\
    n &> \frac{\pi}{\arccos(\frac{2(t+1)^2-1}{2(t+1)^2})} \approx (t+1)\pi\\
    \frac{2t+1}{2t}&>(t+1)\cos(0) - t\cos(\frac{\pi}{m})\\
    \cos(\frac{\pi}{m}) &> \frac{2t^2-1}{2t^2}\\
    m &> \frac{\pi}{\arccos(\frac{2t^2-1}{2t^2})} \approx t\pi
\end{align*}
This is exactly the previously discovered upper bound.
\end{proof}

However, note that we assumed that there exists such a $\theta^*$, but this may not always be the case for $n>m$. If we instead use the same number of prototypes for each circle (i.e. $m=n$), then we can always set $\theta^* = \frac{\pi}{n}$. This results in a configuration where every circle is assigned $n = \ceil{\frac{\pi}{\arccos(\frac{2(t+1)^2-1}{2(t+1)^2})}} \approx \ceil{(t+1)\pi}$ prototypes. While the minimum number of prototypes required on the $t^{th}$ circle remains the same, the \textit{total} minimum number of prototypes required to separate $N$ circles is higher as each smaller circle is assigned the same number of prototypes as the largest one.

\begin{corollary}[Upper Bound - Same Number of Prototypes on Each Circle]
The minimum number of prototypes required to perfectly separate $N$ concentric circles is bounded from above by approximately $N^2\pi$, if each circle must have the same number of assigned prototypes.
\end{corollary}

\subsection{Lower bounds}
An advantage of our formulation of the problem is that it also enables us to search for lower bounds by modifying the $\theta^*$ parameter. We can investigate the scenario where a $\theta^*$ is selected that simultaneously maximizes $d^*_1(m,n,\theta^*)$ and $d^*_1(m,n,\theta^*)$.
\begin{theorem}[Lower Bound]
The minimum number of prototypes required to perfectly separate $N$ concentric circles is bounded from below by approximately $\sum_{t=1}^Nt^{\frac{1}{2}}\pi$, if each circle must have a different number of assigned prototypes.
\end{theorem}
\begin{proof}
If $m\ne n$, the best case would be a $\theta^*$ such that $\cos(\frac{2\pi i}{m} - \frac{2\pi j +\pi}{n} - \theta^*) = \cos(\frac{2\pi i+\pi}{m} - \frac{2\pi j }{n} - \theta^*) = cos(\frac{\pi}{n})$.  Solving the inequalities leads to the following values for $m$ and $n$.  
\begin{align*}
    n &> \frac{\pi}{\arccos(\frac{2t+1}{2(t+1)})} \approx (t+1)^{\frac{1}{2}}\pi\\
    m &> \frac{\pi}{\arccos(\frac{2t^2-t-1}{2t^2})} \approx \frac{t}{(t+1)^{\frac{1}{2}}}\pi
\end{align*}
We note again that such a $\theta^*$ may not always exist. 
\end{proof}

\subsection{Exact and approximate solutions}

In the case where $m=n$, we can always choose a $\theta^*$ such that $\cos(\frac{2\pi i}{m} - \frac{2\pi j +\pi}{n} - \theta^*) = cos(\frac{\pi}{n})$. Solving the inequalities, we get that $n> \frac{\pi}{\arccos(\frac{2t+1}{2(t+1)})} \approx (t+1)^{\frac{1}{2}}\pi$. Thus we have a tight bound for this case.

\begin{corollary}[Exact Solution - Same Number of Prototypes on Each Circle]
The minimum number of prototypes required to perfectly separate $N$ concentric circles is approximately $N^{\frac{3}{2}}\pi$, if each circle must have the same number of assigned prototypes.
\end{corollary}

When $m > n$, we have that $\cos(\frac{2\pi i}{m} - \frac{2\pi j +\pi}{n} - \theta^*) > cos(\frac{\pi}{n})$ as $\frac{2\pi i}{m} - \frac{2\pi j}{n} = \frac{2\pi c_1 \gcd(m,n)}{mn},\; c_1 \in \mathbb{N}_0$. Let $q:= \frac{2\pi \gcd(m,n)}{mn}$, then $|\frac{2\pi i}{m} - \frac{2\pi j +\pi}{n} - \theta^*| \le \frac{q}{2}$ and $|\frac{2\pi i+\pi}{m} - \frac{2\pi j }{n} - \theta^*| \le \frac{q}{2}$. Thus $\cos(\frac{2\pi i}{m} - \frac{2\pi j +\pi}{n} - \theta^*) \ge cos(\frac{q}{2}), \text{ and } \cos(\frac{2\pi i+\pi}{m} - \frac{2\pi j }{n} - \theta^*) \ge cos(\frac{q}{2})$. \\
Using the series expansion at $q=0$ we can find that $\cos(\frac{q}{2})=1 - \frac{q^2}{8} + \frac{q^4}{384} - \frac{q^6}{46080} + O(q^8)$. 

\begin{theorem}[First Order Approximation - Different Number of Prototypes on Each Circle]
\label{lem:order1}
The minimum number of prototypes required to perfectly separate $N$ concentric circles is approximately $1+\sum_{t=1}^Nt\pi$, if each circle must have a different number of assigned prototypes.
\end{theorem}
\begin{proof}
For a first order approximation, we consider $\cos(\frac{q}{2})=1 - \frac{q^2}{8} + O(q^4)$ and $\cos(\frac{\pi}{n})=1-\frac{\pi^2}{2n^2}+O(\frac{1}{n^4})$. Inequality~\ref{eqn:1} then becomes the following. 

\begin{align*}
    -\frac{2t+1}{2(t+1)} &> t(1 - \frac{q^2}{8} + O(q^4)) - (t+1)(1-\frac{\pi^2}{2n^2}+O(\frac{1}{n^4}))\\
    &=-1 - \frac{\pi^2}{2n^2}(t\frac{\gcd(m,n)^2}{m^2}-t-1) +O(\frac{1}{n^4})\\
    n^2&>-\pi^2(t+1)(t\frac{\gcd(m,n)^2}{m^2}-t-1) +O(\frac{1}{n^2})
\end{align*}
However, we know from our previous upper bound that $m+1\le n\le m+4$.\\ Thus $\frac{4}{(n-4)^2}>\frac{\gcd(m,n)^2}{m^2} > \frac{1}{(n-1)^2}$ which means that $\frac{\gcd(m,n)^2}{m^2} = O(\frac{1}{n^2})$. 

\begin{align*}
    n^2&>-\pi^2(t+1)(t\frac{\gcd(m,n)^2}{m^2}-t-1) +O(\frac{1}{n^2})\\
    &= \pi^2(t+1)^2 + O(\frac{1}{n^2})
\end{align*}
Therefore we have that $n+O(\frac{1}{n})>(t+1)\pi$ as desired.
\end{proof}

\begin{figure}
    \centering
    \includegraphics[width=0.8\textwidth]{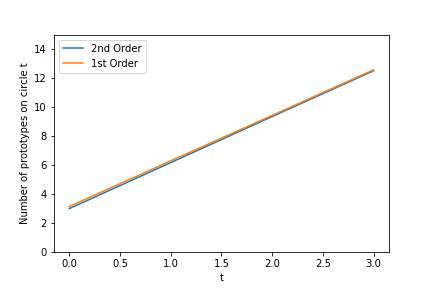}
    \caption{Approximations for the minimal number of prototypes that must be assigned to circle $t$.}
    \label{fig:approximations}
\end{figure}
We plot the second order approximation alongside the first order approximation from Lemma~\ref{lem:order1} in Figure~\ref{fig:approximations} to show that the two quickly converge. Thus we can be confident that approximately $t\pi$ prototypes are required for the  $t^{th}$ circle as this approximation quickly approaches the true minimal number of required prototypes as $t$ increases. Since we can only assign a positive integer number of prototypes to each circle,  we assign $\ceil{t\pi}$ prototypes to the $t^{th}$ circle. Applying this to the initial condition that the $0^{th}$ circle is assigned exactly one prototype results in the following sequence of the minimal number of prototypes that must be assigned to each circle. We note that the sequence generated by the second order approximation would be almost identical, but with a 3 replacing the 4.
\begin{align*}
    1,4,7,10,13,16,19,22,26,29,32,35,38,41 \ldots
\end{align*} 
\begin{corollary}[Approximate Solution - Different Number of Prototypes on Each Circle]
\label{thm:approx}
The minimum number of prototypes required to perfectly separate $N$ concentric circles is approximately $\sum_{t=1}^N\ceil{t\pi}\approx\frac{N+N(N+1)\pi}{2}$, if each circle must have a different number of assigned prototypes.
\end{corollary}

\section{Computational Results}
\begin{figure}
    \centering
    \includegraphics[width=0.32\textwidth]{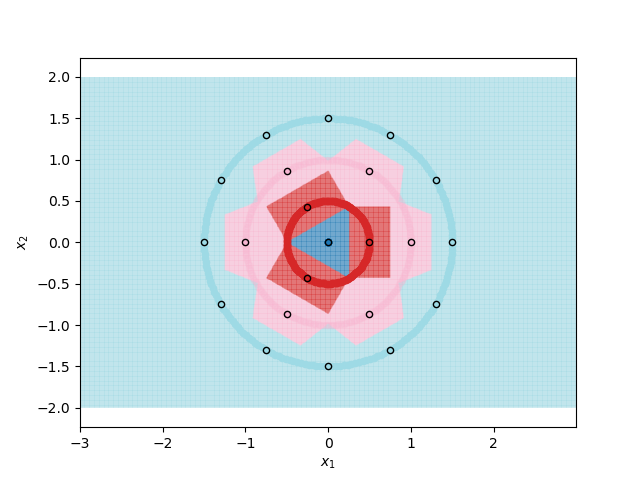}
    \includegraphics[width=0.32\textwidth]{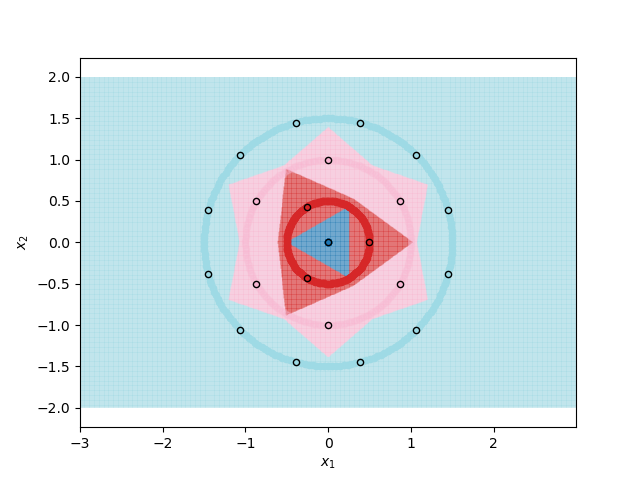}
    \includegraphics[width=0.32\textwidth]{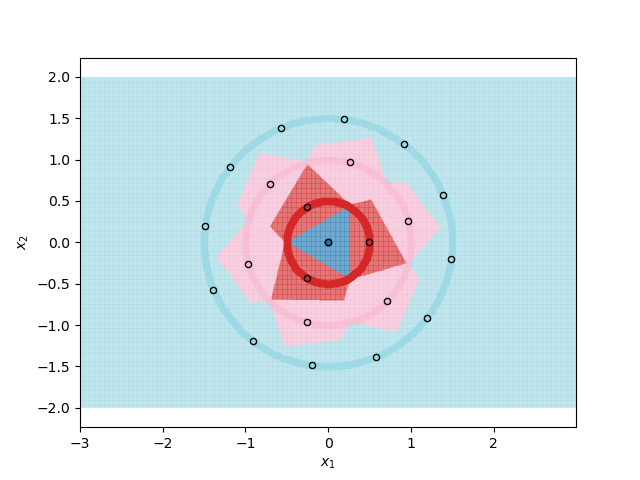}
    \caption{1NN decision boundaries when fitted on prototypes found using the FindPUGS algorithm. Each shaded circle represents a different class and the outlined points represent the assigned prototypes. The colored regions correspond to the decision boundaries created by the 1NN classifier. \textbf{Left and Center}: Prototypes on adjacent circles are not optimally rotated resulting in imperfect class separation in certain regions. \textbf{Right}: Prototypes are optimally rotated resulting in perfect class separation.}
    \label{fig:4circles_alt}
\end{figure}
We can use computational methods to empirically find the minimum number of required prototypes and see if it matches our theoretical results. We propose an iterative algorithm, Algorithm~\ref{algo:findpugs}, that proceeds from the innermost circle to the outermost one finding a near-optimal number of required prototypes in a greedy manner. Our code for this algorithm can be found at the publicly available GitHub repository associated with this paper. We note that it was shown above that the choice of $c>0$, the constant length by which the radius of each consecutive circle increases, does not affect the number of required prototypes. Nonetheless, we still include $c$ as a parameter in our algorithm to verify correctness. Running the algorithm for some large $T$, with any choice of $c$, results in the following sequence.
\begin{align*}
    1,3,6,12,13,16,19,22,26,29,32,35,38,41 \ldots
\end{align*}
This sequence appears to very quickly converge to the sequence predicted by our theorem. Curiously, while there are small differences between the first few steps of the two sequences, these differences cancel out and the cumulative number of required prototypes is identical when there are four or more circles. 
\begin{algorithm}
\SetAlgoLined
\SetKwFunction{algo}{FindPUGS}\SetKwFunction{proca}{d1}\SetKwFunction{procb}{d2}
\KwResult{An ordered list, \textbf{N}, of the minimum number of prototypes required for each circle.}
 $T \leftarrow$ the number of circles\;
 $c \leftarrow$ the length by which radii should grow\;
\SetKwProg{myalg}{Algorithm}{}{}
\myalg{\algo{T, c}}{

 $\textbf{N} \leftarrow$ [1]\;
 \For{$t = 1,2,\ldots,T-1$}{
  $m\leftarrow \textbf{N}[-1]$\;
  $n\leftarrow m+1$\;
  $p \leftarrow 0$\;
  \While{True}{
    $d_a \leftarrow \sqrt{2t^2c^2 - 2t^2c^2\cos(\frac{\pi}{m})}$\;
    $d_b \leftarrow \sqrt{2(t+1)^2c^2 - 2(t+1)^2c^2\cos(\frac{\pi}{n})}$\;
    \For{$i = 0,1,\ldots,4mn$}{
        $\theta \leftarrow \frac{i\pi}{m*n*16}$\;
        \If{d1($t,c,m,n,\theta$) $> d_b$ \textbf{and } d2($t,c,m,n,\theta$) $> d_a$}{
            $p \leftarrow n$ \;
            \textbf{break}\;
        }
    }
    \If{$p>0$}{
     \textbf{N}.append($p$)\;
     \textbf{break}\;

    }
    $n\leftarrow n+1$\;
  }
  }
  \KwRet \textbf{N}\;}{}
 \setcounter{AlgoLine}{0}
  \SetKwProg{myproc}{Procedure}{}{}
  \myproc{\proca{$t,c,m,n,\theta$}}{
    dists $\leftarrow$ [] \;
    \For{$i=0,\ldots,m-1$}{
      \For{$j=0,\ldots,n-1$}{
        dist $\leftarrow \sqrt{t^2c^2 +(t+1)^2c^2 -2t(t+1)c^2\cos(\frac{2i\pi}{m} - \frac{2j\pi}{n} -\frac{\pi}{n} - \theta)}$\;
        dists.append(dist)\;
      }
    }
  \KwRet $\min$(dists)\;}
  \setcounter{AlgoLine}{0}
  \SetKwProg{myproc}{Procedure}{}{}
  \myproc{\procb{$t,c,m,n,\theta$}}{
    dists $\leftarrow$ [] \;
    \For{$i=0,\ldots,m-1$}{
      \For{$j=0,\ldots,n-1$}{
        dist $\leftarrow \sqrt{t^2c^2 +(t+1)^2c^2 -2t(t+1)c^2\cos(\frac{2i\pi}{m} - \frac{2j\pi}{n} +\frac{\pi}{m} - \theta)}$\;
        dists.append(dist)\;
      }
    }
    \KwRet $\min$(dists)\;}
 \caption{FindPUGS Algorithm: Finding (nearly-optimal) Prototypes Using Greedy Search}
 \label{algo:findpugs}
\end{algorithm}
\section{Conclusion}
The kNN classifier is a powerful classification algorithm, but the computational cost can be prohibitively expensive. While numerous prototype methods have been proposed to alleviate this problem, their performance is often strongly determined by the underlying geometry of the data. Certain pathological geometries can result in especially poor performance of these heuristic algorithms. We analyzed one such extreme setting and demonstrated that analytical methods can be used to find optimal prototypes for training a 1NN classifier. We also proposed an algorithm for finding nearly-optimal prototypes in this setting, and used it to validate our theoretical results. Identifying and studying further pathological geometries in kNN and other machine learning models is an important direction for understanding their failure modes and improving training algorithms and prototype methods.

\bibliography{sample}

\end{document}